\documentclass{article} 

\usepackage{amsmath,amsthm,amssymb,txfonts,pifont}
\usepackage{bbding} 
\usepackage{graphicx}
\usepackage{url} 
\usepackage{enumitem}

\newcommand{\rimp}{\Rightarrow} 
\newcommand{\dimp}{\Leftrightarrow} 

\newcommand{\says}{\strictif} 
\newcommand{\Rsays}{R_\strictif} 

\newcommand{\asays}{~\mathtt{says} ~} 

\newcommand{\signs}{~\mathit{sig}~} 
\newcommand{\Rsigns}{R_{\mathit{sig}}}

\newcommand{\entails}{\rightarrowtail} 
\newcommand{\Rentails}{R_\entails}

\newcommand{\Nat}{\mathbb{N}}

\newcommand{\terms}{{\cal T}}
\newcommand{\Agents}{\mathcal{A}}
\newcommand{\Prop}{{\mathcal P}\mathit{rop}}
\newcommand{\formulas}{{\cal F}}

\newcommand{\powerset}[1]{{\cal P}(#1)}

\newcommand{\this}{\tau} 
\newcommand{\Var}{\mathcal{V}\mathit{ar}}

\newcommand{\edepth}{D}

\newcommand{\commentout}[1]{}
\newtheorem{proposition}{Proposition}

\theoremstyle{definition}
\newtheorem{example}{Example}

\newenvironment{oldtheorem}[1]
  {\begin{renewcommand}{\theproposition}{\ref{#1}}}
  {\end{renewcommand}\addtocounter{proposition}{-1}}

\newcommand{\Opterms}{\mathcal{O}}
 
\newcommand{\height}{\mathit{height}}
\newcommand{\depth}{\mathit{depth}}

\newcommand{\create}{\mathtt{create}}
\newcommand{\call}{\mathtt{call}}

\title{A Formal Treatment of Contract Signature} 
\author{Ron van der Meyden\thanks{ R. van der Meyden is with the School of Computer Science and Engineering, UNSW Sydney. 
E-mail: R.VanderMeyden@unsw.edu.au. 
This paper has been accepted to IEEE Transactions on Services Computing. 
The final published version is DOI 10.1109/TSC.2021.3101833. 
\copyright 2021 IEEE. Personal use of this material is permitted. Permission from IEEE must be obtained for all other uses, in any current or future media, including reprinting/republishing this material for advertising or promotional purposes, creating new collective works, for resale or redistribution to servers or lists, or reuse of any copyrighted component of this work in other works.
}}

\begin{document}

\maketitle

\begin{abstract}  
The paper develops a logical understanding of processes for signature of legal contracts, motivated by
applications to legal recognition of smart contracts on blockchain platforms. 
A  number of axioms and rules of inference are developed that can be used to justify a ``meeting of the minds'' 
precondition for contract formation from the fact that certain content has been signed. 
In addition to an ``offer and acceptance'' process, the paper considers
``signature in counterparts'', a legal process that permits a contract between two or more parties to be brought into force by having the parties 
independently (possibly, remotely) sign different copies of the contract, rather  than placing their signatures on a common copy at a physical meeting.
It is argued that a satisfactory account of signature in counterparts benefits from a logic with syntactic self-reference.  The axioms used are supported by a formal semantics, and a number of 
further properties of the logic are investigated. In particular, it is shown that the logic implies that when a contract has been signed, the 
parties do not just agree, but are in \emph{mutual agreement} (a common-knowledge-like notion) about the terms of the contract. 
\end{abstract}

\section{Introduction} 

The idea that one should use formal logic to deal with aspects of legal reasoning has a history dating as far back as work of Leibniz  in the 17th Century \cite{APS13}.
Modern forms of logic and formal representation were applied to reasoning about insurance contracts in the 1930's \cite{Pfeiffer50}, advocated for 
legal drafting in the 1950's \cite{Allen57} and given computational support from the 1970's \cite{Sprowl}. 
Logical representation of contractual relationships between business partners has also been a focus of research 
since the early days of electronic commerce \cite{Kimbrough}.  

A more recent incarnation of formality in legal reasoning is smart contracts in the context of cryptocurrency and  blockchain platforms \cite{Szabo98,ethereum}. 
In their present form, smart contracts generally consist of programs and data that are used to enforce security properties of multi-agent protocols. 
While smart contracts may not be legally enforceable, often these protocols are for processes for which trust amongst the agents would typically have been provided through the use of legal contracts, and many of the applications under consideration, including financial derivatives, tokens representing corporate equity rights, insurance, and loans, are plainly within the scope of contract law. The term ``smart legal contract" is beginning to be applied to contracts that both use computational elements and have legal validity. 
The emergence of the area of smart contracts  therefore gives renewed motivation to study the formal representation of legal reasoning and legal processes.  

One of the novelties of smart contracts, compared to earlier work on formal contract representation, is their application of cryptographic techniques, and in particular, digital signatures. 
In the present paper, we consider questions of logical representation pertinent to the legal process of contract signature.
Our longer term goal for research in this area is to develop a level of abstraction, intermediate between natural language contracts and smart contract code, 
that enables the content of a contract to be  expressed in logical form.  Representations at this level of abstraction would help to bridge between the declarative form of legal 
contracts and the imperative form of smart contract code, and provide a formal specification against which the code can be verified. (Indeed, we believe that there are situations
where the logical representation and the code can be identified, yielding a type of logic programming approach to smart contracts.)

Depending on the legal jurisdiction, various criteria are applied in law in order to determine whether a contract has been validly formed between two 
or more parties. For example, Anglo-American law, or Common law, applies criteria including  ``meeting of minds'' (which may be witnessed by ``offer and acceptance''), 
``consideration" (exchange of value), ``intention to create legal relations", and ``capacity'' (being of requisite age and of sound mind, or holding a corporate position delegated to 
enter into contracts on a company's behalf). For some specific types of transactions, e.g., sale of land, signed documentation is mandatory, but 
in general, a signed agreement is not required for a contract to be established. However,  the formality of signed documentation is very frequently used to help establish the 
evidentiary basis for formation of a contract. Our concern in this paper is specifically with the signature process, rather than with the complete set of legal criteria for formation of a contract. 

In formation of legal contracts using signed agreements, all parties to the contract are required to sign 
in order for the contract to be considered valid.  Frequently, this is done at a physical meeting
of the parties so that copies of the contract can be signed and immediately exchanged for co-signature. 
There are several motivations for this process. For one, it enables the parties signing to be authenticated, and allows for 
witnessing of the signatures.   It is also frequently desirable to 
establish a state of common knowledge amongst the parties that the contract has been signed and that the signers were authenticated: 
a physical signing ceremony achieves this goal. Finally, it prevents one party gaining advantage 
by presentation of a partially signed contract to a third party (e.g., Bob induces Carol to offer a higher price on Bob's house
by showing her the sales contract that Alice has signed).  

However, physical meetings present the difficulties of scheduling of the participants and 
travel costs. In practice, therefore, the parties frequently allow the contract to be considered valid  
when each of the parties has signed a distinct copy. This is referred to as the document being 
signed ``in counterparts", and is considered legally valid in many jurisdictions. 
In some cases there is the additional requirement that the 
contract is not valid until the signed copies have been delivered to the parties. 

The main question we address in this paper is the following: when a party signs their copy of a contract,  
just what, logically, is the attitude that they are taking in doing so? They are generally not assenting that they are bound by the terms $\phi$ of the contract, since that 
depends on the other party or parties to the contract also signing. 
A better characterization would seem to be the conditional assertion ``I assent to be bound by the terms $\phi$ provided that the others do also". 

To formalize this intuition, we work in the setting of a modal logic in the spirit of logics of access control and authentication \cite{Abadi08}, which have been applied to formal reasoning about 
cryptographic protocols, digital certificate infrastructures and access control policies. In particular, we use a modality $A \says \phi$ to capture 
that an agent $A$ ``assents to''  a formula $\phi$.  In the computer security literature, the corresponding modality is usually glossed as capturing what agent $A$ ``says''. 
We prefer the readings ``assents to'' or "agrees that", since in our application of contract signature, the formula $\phi$ will express the terms of a contract, and  what an agent says may have implications for their legal commitments, so has a more formal connotation.  
Since our development requires a distinction between this modality and the syntactic form of content explicitly signed by an agent, the 
logic uses additional constructs $A \signs t$ (agent $A$ has signed syntactic content $t$) and $t\entails \phi$ (syntactic content $t$ entails, or includes in its meaning, formula $\phi$). 
The construct $t\entails \phi$ is used to bridge between syntactic content signed and what an agent assents to. 

We take it to be a key criterion for formation of a contract, expressed in the formula $\phi$, between two parties $A$ and $B$, that both parties agree to the terms $\phi$, 
which we can express in the logic as $(A \says \phi) \land (B \says \phi)$.  This could be understood as corresponding to the legal notion of a `meeting of the minds' of the parties. 
In highly formalized settings such as smart contracts on a blockchain under a ``Code is Law'' interpretation, this condition may be taken to be necessary and sufficient 
for formation of an enforceable contract. However, as noted above,  the law takes a significantly more nuanced view that applies additional, jurisdiction dependent, criteria. 
If $A$ is a child, for example, the law may hold that any promises apparently expressed using $\phi$ do not give rise to enforceable obligations. 
To allow for interpretations in which this condition is used merely as a necessary condition for contract formation, we therefore do not interpret the construct $A \says \phi$ as carrying
normative meaning. The intuitive reading of $A \says \phi$ is therefore weaker than  normative notions such as ``obligation'' \cite{DeonHDBK} and ``commitment'' \cite{BKYD15}
that have been the focus of 
work in deontic logic and multi-agent systems. In our intended application, such normative content could be expressed in the formula $\phi$ itself, though we do not attempt to 
develop the expressiveness required for this in the logic of the present paper. 

We show that the logic can explain a meeting of the minds in an offer and acceptance process for contract formation 
by having the offeror sign a message that states essentially the conditional
``I assent to be bound by the terms $\phi$ provided that you sign $\phi$", to which the 
acceptor responds by signing $\phi$. However, the content signed in this process is asymmetric.  
We show that a naive way to capture signature in counterparts, 
in which the parties sign symmetric conditional statements,  does not  suffice to establish a meeting of the minds. 
We argue that a better understanding can be obtained by treating the contract as making a self-referential 
statement: ``\emph{This contract} may be signed in counterparts", that one does indeed find in the natural language text of many actual contracts. 
The problem that then arises is how to make formal sense of such self-reference, given that attempts to introduce self-reference into logic are fraught with paradox \cite{sep-self-reference}. 
We solve this problem by developing a logical treatment that  allows self-reference without falling into inconsistency.

Key to our approach to handling self-reference without paradox  is the distinction between 
syntactic terms and their logical entailments. 
In addition to naturally fitting the underlying cryptography, this enables our semantics of self-reference to avoid 
complex constructions using three-valued logic, fixed points or nonstandard set theories. 

We give a number of axioms and a model theoretic semantics that validates these axioms. 
We then show that the axioms allow a formal account of the reasoning by which a contract signed in counterparts becomes valid. 
We go on to study some further properties of the logic. Our logic uses an axiom, similar to others in the literature, that states essentially that 
if an agent has signed a message $t$, then all agents assent to  the fact that it has signed this message. We show that 
it follows that our account of contract signature implies not just that the parties jointly assent to the terms of the contract, but that they 
\emph{mutually assent}, a stronger common-knowledge like property. Indeed, we show that not just the agents, but society itself mutually assents to the 
fact that the agents mutually assent.  This is a much stronger conclusion, that may be questionable in the 
context of asynchronous or  unreliable communication. However, we argue that the conclusion is justifiable under some interpretations of the logic that involve use of 
trusted third parties or a blockchain to register the signatures. 

The focus of these modelings of contract formation is on a declarative representation of contracts using formulas of a logic. 
In current practice, on platforms such as Ethereum, smart contracts are not represented declaratively, but as imperative code and data. 
We also show that the logic provides a number of ways to formally model the ascription of declarative meaning to smart contracts represented in such an imperative form: 
either by formation of a separate contract that describes how the smart contract is to be interpreted, 
or by a jurisdiction standardizing a declarative interpretation for blockchain messages. 
In our modeling for the latter case, the entailment relation $t\entails \phi$ is used to associate 
declarative consequences $\phi$ of the messages $t$ that are signed by participants in a smart contracts, 
such as the message that creates a smart contract on the blockchain, and a message that makes function call to a smart contract. 
In particular, we show how taking particular formulas to be entailed by such messages can explain a meeting 
of the minds of the creator of a smart contract and a participant in the smart contract.  

The structure of the paper is as follows. Section~\ref{sec:logic} introduces the syntax of a logic 
dealing with signed messages, their semantics, and the consequences for what an agent assents to.
 We give this logic a model theoretic semantics in Section~\ref{sec:sem}. In Section~\ref{sec:offer}, we show that the logic can be used
 to give an account of an offer and acceptance process for contract signature. Section~\ref{cpt1} turns to the topic of signature in counterparts. 
 An abstract account of how the individual signatures of the parties leads to their joint assent to the terms of the contract
is provided. However, this account relies upon an unexplained assumption about the meaning of the contract. To justify this assumption, 
we then turn in Section~\ref{sec:self-ref} to extending the syntax and semantics of the logic to include self-referential formulas.  We show that this can be 
achieved without falling into contradiction. Section~\ref{cpt2} returns to signature in counterparts, showing that the previously unexplained assumption 
can be justified by taking the contract to be self-referential.  Section~\ref{sec:common} deals with the issue of mutual assent to the contract. 
Our primary focus in the paper is towards a declarative view of contracts, but we discuss how the logic might be applied to associating a declarative interpretation to 
imperative smart contracts, as in the current practice, in Section~\ref{sec:imperativesc}. 
Section~\ref{sec:related} discusses related work. 
Finally, Section~\ref{sec:concl} concludes with a discussion of possible future research directions. 
An appendix gives proofs omitted from the body of the paper. 

\section{A Logic} \label{sec:logic}

The logic can be understood as describing a static situation, in which it has been 
determined which messages have been signed, and all these signed messages are available to all agents. 
We would like to determine what each agent will be understood to have formally agreed to in such a situation. 
The reader may find it helpful to think of the logic as describing a situation in which all the relevant information about what has been 
signed has been presented in court. (Other interpretations of the logic are discussed in Section~\ref{sec:common}.) 

The syntax of the logic is parameterized by a tuple $\Sigma = (\Agents, \Prop, \Opterms)$ where 
$\Agents$, $\Prop$ are disjoint sets and $\Opterms= (\Opterms^0, \Opterms^1, \ldots)$ 
is a sequence of sets $\Opterms^n$, also disjoint. Intuitively,  
$\Agents$ is a set of atomic terms representing agents, with generic elements $A,B, \ldots$. 
The set $\Prop$ is a set of atomic terms, representing atomic propositions, with generic elements $p,q\ldots$ 
The set $\Opterms^n$ for $n \in \Nat$ contains operator names, understood to have \emph{arity} $n$. 
A generic element of $\Opterms^n$ is written $o^n$ to indicate that operator $o$ has arity $n$. Given $\Sigma$, we define a set of terms $\terms$, with generic element $t$ 
and a set of formulas $\formulas$, with generic element $\phi,\psi, \ldots$.  
Formally, terms and formulas are specified by
$$ \begin{array}{l} 
t ::= A ~|~ o^n(t_1, \ldots, t_n)~|~ \phi \\
\phi ::= p~ |~ \neg \phi ~|~ \phi \land \phi~ |~   t \entails \phi~ |~ A \signs t~ |~ A\says \phi 
\end{array} 
$$
where $n \geq 0$ and $o^n$ is any operator in $\Opterms^n$ and 
$t_1, \ldots t_n$ are terms in $\terms$.
Note that $\formulas \subset \terms$, so every formula is also a term. 
Intuitively, terms not in $\formulas$ 
represent application specific content that is not purely logical, but may still be signed and may contain formulas as subterms. 
Boolean constructs other than the two included, such as  $\phi_1 \rimp \phi_2$ and $\phi_1 \lor \phi_2$, 
can be treated as abbreviations for formulas in the language 
in the usual way. 

 Atomic propositions $p$ are intended to represent assertions such as ``Alice has the obligation to pay Bob 30  Ethers by Dec 7, 2021.'' 
We envisage extensions of the logic tailored to representation of the content of contracts, and that this 
will involve a richer base logic of formulas, including quantifiers, action expressions and temporal and deontic operators. Since the 
present paper is concerned primarily with the signature process, this richer expressiveness has been 
abstracted to the set of atomic propositions. 

Intuitively,  $ t \entails \phi$ expresses that term $t$ ``entails'' formula $\phi$. 
In general, terms may represent both logical and non-logical content. For example, 
a term representing a contract may contain non-logical information such as a date of creation, the names of the parties, 
as well as logical content in the form of clauses that capture the consequences of the contract.  
The latter could correspond to formulas $\phi$ such that~$t \entails \phi$. 
The precise semantics of $t \entails \phi$ will be application specific. One application might include using 
$t$ to represent the (controlled) natural language text of a legal contract, and $\phi$ to represent its 
content in logical form as a specification of a smart contract. Alternately, $t$ might express a standard Electronic Data Interchange message 
in the form of a set of attribute-value pairs, and $\phi$ its intended logical semantics. 

The formula $A \signs t$  expresses that agent $A$ has ``signed" term $t$. 
Intuitively, this means that $A$ has applied one of their private signature keys to (a serialisation of) 
the term $t$, and that other parties who know the corresponding  public verification key is associated
to $A$ can verify that the signature is valid. 
Authentication of $A$ here might be simply because identity $A$ is semantically represented as 
identical to the public key, or because the association of $A$ to the public verification key is 
attested by a trusted certification authority. In the present paper,  the logic abstracts from such details. 
Note that we permit an arbitrary term to be signed, not just a formula. 

Finally, $A\says \phi$  expresses that agent $A$ ``assents to" or ``agrees to" formula $\phi$. Intuitively, this means that 
$A$ agrees to $\phi$ and its logical consequences. Typically, this will be because there exists evidence 
in the form of (cryptographically) signed content, using which, such agreement can be proved. 
In particular, if $A$ has signed a message that means (entails) $\phi$, it will follow that $A$ agrees to $\phi$. 
However, $A$ will also have to agree to facts that cannot reasonably be disputed, such as facts about what content other agents have signed. 
As noted above, the logic can be understood as dealing with a situation in which all signed content is available to all agents. 
We discuss possible interpretations of this operator at greater length in Section~\ref{sec:common}. 

The logic has the following axiom schemas%
\footnote{Note that the logic does not include quantification, but we get the effect of universal quantification from the fact that
every well-typed substitution instance of the schemas is an axiom.}  
and rules of inference. In the following, $\phi, \psi$ are formulas, $t$ is a term and $A,B$ are agents.
We write $\vdash \phi$ to mean that $\phi$ is derivable from axioms using the rules of inference given. 

~\\
\noindent
{\bf  Axioms:} \begin{enumerate}[label=\textbf{Ax\arabic*}]
\item \label{ax:prop} All substitution instances of tautologies of propositional logic 
\item \label{ax:fomulaentails} ~~~ $\phi \entails \phi$ 
\item \label{ax:entailsclosure} ~~~ $((t \entails \phi) \land (t \entails (\phi \rimp \psi))) \rimp (t \entails \psi)$ 
\item   \label{ax:signssays} ~~~ $((A \signs t)\land (t \entails \phi))\rimp A \says \phi$  
\item \label{ax:saysnormal} ~~~$(A \says \phi )\land (A \says (\phi \rimp \psi)) \rimp A \says \psi $
\item \label{ax:liftsigns} ~~~$  (B \signs t) \rimp A \says (B\signs t) $
\item \label{ax:liftentails} ~~~$  (t \entails \phi) \rimp A \says (t \entails \phi) $
\end{enumerate} 

\noindent 
{\bf Rules of Inference:} 
\begin{enumerate}[label=\textbf{R\arabic*}]
\item \label{r:MP} $\vdash \phi$ and $\vdash \phi \rimp \psi$ implies $\vdash\psi$. 
\item \label{r:Nentails} $\vdash \phi$  implies $\vdash t \entails \phi$. 
\item \label{r:Nsays} $\vdash \phi$  implies $\vdash A \says \phi$. 
\end{enumerate}

Note that axiom~\ref{ax:entailsclosure} and rule~\ref{r:Nentails} together state that ``$t\entails$'' is a normal modal operator for each term $t$. 
Similarly axiom~\ref{ax:saysnormal} and rule~\ref{r:Nsays} together state that ``$A\says$'' is a normal modal operator for each agent $A$. 
Axiom~\ref{ax:fomulaentails} says that a formula (as a term) entails itself. (Entailments of non-formula terms are application specific and are not 
constrained by the  logic.) 
Axiom~\ref{ax:signssays} says that if agent $A$ has signed $t$ then they assent to all entailments of term $t$. 
Axiom~\ref{ax:liftsigns} can be understood as stating that signed messages are indisputable, in the sense that if agent $B$ has
signed $t$ then agent $A$ must agree that $B$ signed $t$ --- agent $A$ is unable to deny that the signature exists. 
Finally, Axiom~\ref{ax:liftentails} states that agents assent to all facts about entailment; intuitively, this captures that all agents
are in agreement about the meaning of terms.

\section{Semantics} \label{sec:sem}

The logic can be given a Kripke style semantics as follows. Given the parameters $\Sigma = (\Agents, \Prop, \Opterms)$ where 
$\Agents$ is the set of  agent names, $\Prop$ is the (disjoint) subset of  atomic propositions, and $\Opterms$ is the ranked set of operators,  
the language is defined by a set of terms $\terms$, and a set of formula $\formulas$. 
A model for the language based in these parameters  is a tuple $\langle W, \Rsigns, \Rentails, \Rsays , \pi \rangle $, 
where  the components and their intuitive interpretations are as follows: 
\begin{itemize} 
\item $W$ is a set,  whose elements are called \emph{worlds}, 
\item $\Rsigns \subseteq W \times \Agents  \times \terms$ is a relation, such that 
$(w,A,t) \in \Rsigns$ represents that in world $w$, agent $A$ has signed term $t$, 
\item 
$\Rentails \subseteq \terms \times W$ is a relation, such that 
$(t,w) \in \Rentails$ represents that world $w$ is consistent with all the information
entailed by term $t$, 

\item 
$\Rsays \subseteq W\times \Agents \times W$ is a relation, such that 
$(w,A,w') \in \Rsays$ represents that world $w'$ is consistent with all that agent $A$ assents to in world $w$, 

\item $\pi : W  \rightarrow \powerset{\Prop} $ is an interpretation that associates each world with the set of atomic propositions 
holding at the world. 
\end{itemize} 
Note that the relation $\Rentails$ is not relativized to a world. Intuitively, the meaning of terms is independent of the state of the world, and is ``common knowledge" to all agents, who all ``speak the same language".  
We do not assume that, for a fixed world $w$, the set of $w'$ for which $(w,A,w') \in  \Rsays$ is non-empty. Intuitively, we allow that an agent assents to an inconsistency, in which case no worlds are consistent.

The semantics of the logic is given by a relation of satisfaction $M,w \models \phi$, 
where $M$ is a model, $w$ is a world of $M$ and $\phi$ is a formula. This relation is 
defined recursively by 
\begin{itemize} 
\item $M,w\models p$, for $p\in \Prop$, when $p \in \pi(w)$, 
\item $M,w\models \neg \phi$ if not $M,w \models \phi$,
\item $M,w\models \phi_1 \land \phi_2$ if $M,w\models \phi_1$ and  $M,w\models \phi_2$,  
\item $M, w\models A \signs t$ if $(w,A,t) \in \Rsigns$, 
\item $M, w\models t \entails \phi$ if $M,w' \models \phi$ for all $w'\in W$ such that $(t,w') \in \Rentails$, 
\item $M, w \models A\says \phi$ if $M,w' \models \phi$ for all $w' \in W$ such that $(w,A,w') \in \Rsays$. 
\end{itemize} 

A formula $\phi$ is \emph{valid} in a model $M$, written $M \models \phi$, if $M,w \models \phi$ for all worlds $w$ of $M$. 
A rule of inference is valid in a model $M$ if, for all worlds $w$ of $M$, if $M,w\models \alpha$ for
all formulas $\alpha$ in the antecedant of the rule, then $M,w\models \beta$ for the formula $\beta$ in the consequent.

In order to obtain models validating the axioms, we assume that a number of semantic constraints hold: 
\begin{itemize} 
\item[SC1.] For formulas $\phi$, we have $(\phi,w) \in \Rentails$ implies $M,w\models \phi$. 
\item[SC2.] If $(w,A,t)\in  \Rsigns $ then  $(w,A,w')\in  \Rsays $ implies $(t,w') \in \Rentails $.
\item[SC3.] If $(w,B,t)\in  \Rsigns $ and  $(w,A,w')\in  \Rsays $ then $(w',B,t) \in \Rsigns$. 
\end{itemize} 
Intuitively, SC1 says that a term that is also a formula entails that formula itself: every world consistent with what is entailed must satisfy the formula. 
SC2  expresses axiom~\ref{ax:signssays}  semantically: it says that if agent $A$ has signed $t$ then they assent to the entailments of term $t$, in the sense that 
any world $w$ consistent with what $A$ assents to must be consistent with these entailments.
SC3 expresses Axiom~\ref{ax:liftsigns} semantically. It says that if $B$ has signed $t$ in world $w$, then $B$ has also signed $t$ in  any world $w'$ that is consistent with  what $A$ says in world~$w$.

\begin{proposition} \label{prop:sound}
The axiom schemas~\ref{ax:prop}-\ref{ax:liftentails} 
and  rules of inference~\ref{r:MP}-\ref{r:Nsays}
are valid in models satisfying SC1-SC3.
\end{proposition}

We discuss some further axioms that are valid with respect to the semantics in Section~\ref{sec:common}, but we make no attempt in this paper at completeness: our principal concern 
is to develop a minimal set of axioms that support our main focus of reasoning about contract signature processes.  

Condition SC1 may present some difficulties when constructing models, since it must be satisfied for the 
infinite set of formulas, and moreover refers to the semantics of formulas. Because SC2 places a lower bound on $\Rentails$, 
the trivial solution where there are no worlds $w$ such that $(\phi,w) \in \Rentails$ is not satisfactory. However, 
starting with any relation $R^0 \subseteq (\terms\setminus \formulas)\times W$ that expresses the
entailments of terms that are not formulas, it is possible to extend this relation to one that 
expresses the entailments of formulas so as to satisfy condition SC1. 

More precisely, given a model $M = \langle W, \Rsigns, \Rentails, \Rsays , \pi \rangle$  and a relation $R$,
write $M(R)$ for the result of replacing $\Rentails$ by $R$ in $M$, that is, 
$M(R) = \langle W, \Rsigns, R, \Rsays , \pi \rangle$.   
We can then express the extension claim as follows. 

\begin{proposition} \label{prop:extend}
Let $M$ be a model, and let 
$R^0 \subseteq (\terms\setminus \formulas)\times W$. There exists a relation 
 $R^\omega \subseteq \terms \times W$ such that $R^\omega \cap (\terms\setminus \formulas)\times W = R^0$ and
 for all formulas $\phi$ and worlds $w\in W$, we have 
 $(\phi,w) \in R^\omega$ iff $M(R^\omega) ,w\models \phi$. 
\end{proposition} 

As the proof is somewhat technical, using a fixed point construction, it is deferred to an
appendix. (The only place in the body of this paper where we need this result is when constructing a model in 
Example~\ref{ex:cex}.)

\section{Offer and Acceptance} \label{sec:offer}

One of the criteria in law for formation of a contract, and our primary focus in this paper, is a `meeting of the minds' concerning the terms of the 
contract. When the terms of the contract between parties $A$ and $B$ are expressed by the formula $\phi$, we may represent this in our logic as $(A \says \phi) \land (B \says \phi)$, 
i.e., both $A$ and $B$ assent to $\phi$. We take this condition as a necessary outcome of any process used by the parties to enter into 
contractual relations, and enquire into how processes for contract signature meet this condition.  

A common process whereby two parties $A,B$ enter into a contract, accepted in law as demonstrating the 
criterion of a ``meeting of the minds'', is for $A$ to make an offer of the contract terms, and for $B$ to accept. 
When implemented in a network setting, or when the parties require evidence of the communication, we 
expect that both the offer and acceptance will be signed. 

We could attempt to express $A$'s offer of terms $\phi$ in the logic as $A \signs \phi$. By~\ref{ax:fomulaentails} 
and~\ref{ax:signssays}, this implies $A \says \phi$, so $A$ assents to the terms $\phi$.

From $A$'s point of view, this is too strong, since it has the risk that if $B$ does not accept the offer, 
$A$ will remain bound to the terms $\phi$. For example, if $\phi$ expresses 
`$A$ shall pay \$US 100 to $B$ and $B$ shall transfer JPY 10,000 to $A$', then 
 $A \says\phi$ implies that $A$ agrees (amongst other things) that $A$ shall pay \$US 100 to $B$. Party $A$ would not want to 
 be held to account for this apparent promise if $B$ does not accept the offer. 
 
We can also conclude that $A$ agrees that `$B$ shall transfer JPY 10,000 to $A$'. 
 In the absence of a matching agreement by $B$, this is pragmatically somewhat peculiar. 
Party $A$ cannot, in general, make promises on $B$'s behalf, and 
unless $A$ is in a position to issue orders to $B$, a mere statement by $A$ will not 
have the effect of placing $B$ under any obligation. 

One might argue that since the entirety of $A$'s original statement, as signed, is unenforceable if $B$ does not accept the offer, 
no part of it is enforceable. Indeed, contract law would rule that no \emph{contract} exists in this 
circumstance, so even if $A$ has made a promise, no legal action will be taken to enforce it. 
Still, notwithstanding the lack of legal enforcement, $A$ would want to avoid even the appearance of 
having a moral obligation to $B$ if the offer is not accepted. 

To avoid being committed to a promise if $B$ does not accept, $A$ could make their offer 
conditional on $B$'s acceptance. This can be expressed in the logic as 
\begin{equation} 
A \signs ((B \signs \phi) \rimp \phi)\label{eq:condoffer}
\end{equation}
from which we obtain 
$A \says  ((B \signs \phi) \rimp \phi)$, by~\ref{ax:fomulaentails} 
and~\ref{ax:signssays}, as before. 

Now, when $B$ accepts the offer with $B \signs \phi$ (implying that $B \says \phi$), 
we deduce using~\ref{ax:liftsigns} that 
$A\says (B \signs \phi)$. 
From this and (\ref{eq:condoffer}) we obtain using~\ref{ax:saysnormal}, that 
$A \says \phi$. Thus, we have $(A \says \phi) \land (B \says \phi)$, as required for a meeting of the minds. 

Note that, at the time $B$ signs, we already have $A \signs ((B \signs \phi) \rimp \phi)$, 
so $B$ can be  assured that they will be able to hold $A$ to the terms of the contract, 
and the risk to $A$ in signing $\phi$ directly does not apply to $B$. 

It is worth noting that the argument works also with a slightly weaker form of the content signed by 
$A$: 
\begin{equation} 
A \signs ((B \says \phi) \rimp \phi)\label{eq:condoffer2}
\end{equation}
From axioms~\ref{ax:fomulaentails} and~\ref{ax:signssays} we have $\vdash (B \signs \phi) \rimp B \says \phi $.  
Applying rule~\ref{r:Nsays}, we get $\vdash A \says ((B \signs \phi) \rimp B \says \phi) $.  
Thus, once $B \signs \phi$ we derive, as before, $ A \says (B \signs \phi)$, 
and can conclude that  $ A \says (B \says \phi)$, hence $A \says \phi$, 
using~\ref{ax:fomulaentails}, \ref{ax:signssays}, \ref{ax:saysnormal} and   (\ref{eq:condoffer2}). 

The above approaches deal with an offer and acceptance between two parties. When the contract has a larger number of parties, 
some more care is required. Consider a contract between three parties $A,B$ and $C$,  with the contract offered by $A$ and $B$ and $C$ accepting. 
One way to generalize from the two-party case would be with assertions $$A \signs (((B \signs \phi) \land (C \signs \phi)) \rimp  \phi)$$ and 
$B \signs \phi$ and $C \signs \phi$. However, this places both $B$ and $C$ in the situation of assenting to $\phi$ when they are not guaranteed that 
the contract will in fact be formed. A better alternative is to chain the conditional assertions signed, using 
$$A \signs ((
                  (B \signs ((C \signs \phi) \rimp \phi)) \land 
                  (C \signs \phi)
                 ) 
                 \rimp  \phi)$$ and 
 and $B \signs ((C \signs \phi) \rimp \phi)$  and $C \signs \phi$. 
 If $A$ passes their signed message to $A$ and $B$ and 
 $B$ then passes their signed message to $C$, then $C$ can be assured their their unconditional signature will validate the contract. 
 Similarly, $B$ can be assured that the contract will be validated once $C$ signs $\phi$, so $B$'s conditional statement can be safely made.                

This idea can be generalized to $n$ parties $A_1 \ldots A_n$, using formulas 
$\sigma_1, \ldots, \sigma_n$ defined by $\sigma_n = A_n \signs \phi$ 
and $\sigma_{k} = A_{k} \signs ( (\bigwedge_{i=k+1 \ldots n} \sigma_i) \rimp \phi)$
for $k=1 \ldots n-1$. However, this approach is highly asymmetric, and requires, for safety, 
that the signed messages be passed in a linear chain between the agents, with $A_{k+1}$ delaying their signature  
until they have received the signed messages supporting $\sigma_1, \ldots \sigma_k$. 
In the following section, we develop a more symmetric approach to contract signature. 

\section{Representing Counterpart Signatures} \label{cpt1}

As an alternative to the offer and acceptance process, 
we now consider signature in counterparts. 
Suppose that $A$ and $B$, operating in a network setting, 
wish to sign an agreement whose meaning is captured by the 
formula $\phi$. Moreover, unlike the offer-acceptance approach, 
we would like the parties to sign the \emph{same}, or at least \emph{symmetric}  content. 

An approach that does not work is for $A$ and $B$ to independently sign $\phi$, i.e.,
 $A \signs \phi$ and $B \signs \phi$, and then exchange these signatures. 
As noted above, we can then derive $A \says \phi$ and $B \says \phi$, so that 
 both $A$ and $B$ agree to the terms of the contract as soon as they sign. 
For both parties, this has the problem discussed above for the situation in an offer-acceptance process in which the 
offeror signs $\phi$. 
 We do not wish either party to have agreed to the contract until the other also has agreed. 
 
Mirroring the conditional account of offer and acceptance above, 
 we could try to make the version of the document that $A$ signs conditional on $B$ having signed, and \emph{vice versa}: 
 $$ A \signs ((B \signs \phi) \rimp \phi) ~~\text{and}~~B \signs ((A \signs \phi) \rimp \phi) ~.$$
 This will not work, since it still relies upon production of the direct signatures  $A \signs \phi$ and $B \signs \phi$
that we are trying to avoid. (Were we to add one or both of 
these, we would have a redundant form of the previous offer and acceptance process.)  

An alternative is to work with the weaker form of the conditions, as in  
 $$ A \signs ((B \says \phi) \rimp \phi) ~~\text{and}~~B \signs ((A \says \phi) \rimp \phi) ~.$$
and hope that we can then derive $A \says \phi$ and $B \says \phi$. 
Unfortunately, this also does not work. By \ref{ax:signssays}, we can derive 
 $$ A \says((B \says \phi) \rimp \phi) ~~\text{and}~~B \says ((A \says \phi) \rimp \phi) ~.$$
However, this is too weak: these assertions have a model, satisfying our axioms, in which neither 
$ A \says \phi$ nor $B \says \phi$. 

\begin{example} \label{ex:cex}
Suppose $\phi$ is the atomic proposition $\mathtt{p}$, and let 
$M= \langle W, \Rsigns, \Rentails, \Rsays , \pi \rangle $, be a model with $W = \{w_0,w_1\}$, 
$$\Rsigns = \{(w,A,(B \says \phi) \rimp \phi), (w,B, (A \says \phi) \rimp \phi)~|~w\in W\}~,$$
$$\Rsays = (W\times \{A,B\} \times W)~,$$ 
and $\pi$ defined by $\pi(w_0) = \emptyset$ and 
$\pi(w_1) = \{\mathtt{p}\}$. 
By the construction given in Section~\ref{sec:sem}, 
given the relation $R^0 = (\terms \setminus \formulas) \times W$, 
we may construct a relation $\Rentails = R^\omega$ such that 
for all formulas $\psi$, we have $(\psi,w) \in \Rentails$ iff 
$M,w \models \psi$.  (Intuitively, the particular starting point $R^0$ we have selected here 
takes all terms that are not formulas to have only trivial, i.e., valid, entailments.)  

Note that for all worlds $w$, we have  $M,w\models  A \signs ((B \says \phi) \rimp \phi)$, and $M,w \models B \signs ((A \says \phi) \rimp \phi)$. Moreover, for all $w\in W$ 
we have $M, w\models \neg (A \says \phi)$, since $(w, A,w_0) \in \Rsays$ and $M, w_0, \models \neg \phi$. 
Similarly, $M, w\models \neg (B \says \phi)$ for all $w \in W$. 
Hence, all the assumptions of the proposed approach to counterpart signatures hold, but the desired conclusion that $M,w \models (A \says \phi) \land (B\says \phi)$ does not. 

The model $M$ has been constructed to satisfy constraint SC1. We show that it also satisfies the constraints SC2-SC3, 
from which it follows using Proposition~\ref{prop:sound} and that it does not follow using the axioms and rules of inference 
that the conclusion $A \says \phi \land B \says \phi$ can be derived from the conditional signatures. 

For constraint SC2, note that we have $(w,X,t) \in\Rsigns$ iff 
either $X = A$ and $t = (B \says \phi) \rimp \phi)$ or 
$X = B$ and $t = (A \says \phi) \rimp \phi)$. 
We need to show that in these cases, if $(w,X,w')\in \Rsays$ then $(t, w') \in \Rentails$. 
In both cases $t$ is a formula $\psi$  in the form of an implication whose antecedent is false 
at all worlds, and therefore $\psi$ is true at all worlds. Since we have $(\psi,u) \in \Rentails$ iff $M,u \models \psi$, we have
that $(\psi,u) \in \Rentails$ for all worlds $u$ in $W$. It follows that SC2 holds. 

For constraint SC3, note that the model satisfies $(w,X,t) \in\Rsigns$ iff $(w',X,t) \in\Rsigns$ for all worlds $w'\in W$.  
It follows that SC3 holds. 
\qed
\end{example}

As an alternative to relying on the statements made by the parties, our resolution of the problem is to make the contract itself assert that it is valid if signed by both parties. This requires 
allowing the contract to be self-referential. We develop the solution first in the abstract, and propose a specific  concrete syntax and semantics for self-reference  in the next section. 
For our abstract presentation, it suffices to capture self-reference by means of an assumption about the entailment relation $\entails$. 
Let $c$ be a term representing the contract itself, and 
let $\phi$ be a formula capturing the terms of the contract. 
 We assume that the following holds: 
$$ c \entails (((A \signs c) \land (B \signs c)) \rimp \phi)~.$$
Intuitively, this says that the contract $c$ entails that, once both $A$ and $B$ have signed it, 
$\phi$ holds. 

Suppose now that we have $A \signs c$ and $B \signs c$. We show that it is now possible to derive $A \says \phi$ and $B \says \phi$, so that both $A$ and $B$ 
assent to $\phi$.  Note first that from $A \signs c$ and the above assumption,  we have 
$$A \says (((A \signs c) \land (B \signs c)) \rimp \phi)$$
by~\ref{ax:signssays}.  Using \ref{ax:liftsigns} we also have that 
$A \says (A \signs c)$  and $ A\says (B \signs c)$. 
By normality of $\says$, we  deduce $A \says \phi$. A 
similar argument shows $B \says \phi$. 
This establishes the following: 

\begin{proposition} 
$\vdash ((c \entails (((A \signs c) \land (B \signs c)) \rimp \phi)) \land (A \signs c) \land (B \signs c)) \rimp ((A\says \phi) \land (B\says \phi))$
\end{proposition} 

The question now arises as to how we obtain a term $c$ such that $c \entails (((A \signs c) \land (B \signs c)) \rimp \phi)$. 
One possible answer is that we obtain this by \emph{fiat}. The entailment relation $\entails$ is application-specific, so  we could 
introduce a ternary operator $\mathtt{contract\_sic}$ (for ``contract signable in counterparts'') and restrict to models
such that the desired entailment holds for the term $c=\mathtt{contract\_sic}(A,B,\phi)$. For example, the legal system governing the 
contract could establish the convention (e.g., by means of legislation or regulatory ruling) 
that terms of this form have the desired entailment. While this has the desired effect, it leaves the parties dependent on their external environment, 
and it  remains unresolved how they may proceed when operating in an environment that does not have such a convention in place. 
In what follows, we show that by extending the logic with a capability for self-reference, it becomes possible to identify a natural 
formula $c$ that \emph{necessarily} satisfies the desired entailment.

\section{A syntax and semantics for self-reference} \label{sec:self-ref}

We now develop a specific syntax and semantics for self-referential terms. We extend the syntax given 
above. The language is now parameterized by a tuple $\Sigma = (\Agents, \Prop, \Opterms,\Var )$ where 
all components are as above, but we add a set $\Var$ of variables, with generic element $x,y,\ldots$. 
We extend the syntax of terms and formulas by modifying the definition to the following: 
$$ \begin{array}{l} 
t ::= x ~|~A ~|~ o^n(t_1, \ldots, t_n)~|~ \phi \\
\phi ::= p~ |~ \neg \phi ~|~ \phi \land \phi~ |~   t \entails \phi~ |~ A \signs t~ |~ A\says \phi ~|~ \this x . \phi
\end{array} 
$$
Here variables $x$ have been added to the base case for terms $t$. 
There is also a new binary operator $\this$ which, when applied to a variable $x$ and a formula $\phi$, produces a formula 
written $\this x . \phi $\,. 
Note that variables may not may appear in the base case of the recursion for formulas --- only propositions $p$ in $\Prop$ may do so. 
The reason for this restriction is to avoid paradox, as explained below. 
Since variables do appear in the base case for terms $t$, they may thereby may appear in the 
formulas $A \signs t$ and $t \entails \phi$ within the subterms~$t$. 

Intuitively, $\this x . \phi $ says that $\phi$ holds, where, in the context of $\phi$, the variable $x$ refers to the formula  $\this x . \phi $. 
Semantically, we think of the denotation of $x$ as a term, i.e., as pure syntax.

An occurrence of a variable $x$ in a term $t$ is said to be \emph{free} if it is not inside any subterm of $t$ of the form $\this x. \phi$. 
Substitution of a term $t$ for the free occurrences of variable $x$ in a term $u$, denoted $u[x \mapsto t]$, is defined by the 
usual recursion. In particular, $(\this y (\phi))[x\mapsto t] = \this y (\phi)$ when $y=x$ and 
$(\this y (\phi))[x\mapsto t] = \this y (\phi[x\mapsto t] )$ otherwise. 
For all other cases, the definition passes the substitution down to all direct subterms, 
e.g., $(u\entails \phi)[x \mapsto t] = (u[x \mapsto t]) \entails (\phi[x \mapsto t])$. 
We define a formula to be a \emph{sentence} if it has no free variables. 

To extend the semantics, we restrict the application of the satisfaction relation to sentences.%
\footnote{At the cost of adding some complexity by adding an interpretation of variables on the left of the 
relation we could extend this to all formulas, but we will not need this expressiveness for our purposes in this paper.}
(Note that every formula in the previous, more restricted syntax, 
is a sentence, since it contains no variables, so this still encompasses the previous semantics.) 
The semantics is extended by adding to the definition above the case 
\begin{itemize} 
\item $M, w\models \this x. \phi$ if $M,w \models \phi[x\mapsto  \this x. \phi]$. 
\end{itemize} 
That is, $\this x. \phi$ holds if $\phi$ holds, with the term  $\this x. \phi$ substituted for free ocurrences of $x$ in $\phi$. 
In effect, this makes such occurrences equivalent to a reference to the formula $\this x. \phi$. 

This semantics may appear to be viciously recursive, making the interpretation of $\this x. \phi$ depend on the 
semantics of a formula $\phi[x\mapsto  \this x. \phi]$ that may itself contain the subformula  $\this x. \phi$. 
However, we note that the syntactic restrictions adopted prevents this from arising. Recall that the variable $x$ may occur only in terms 
$u$ appearing in subformulas of $\phi$ of the forms $A \signs u$ for some agent $A$, 
or $u\entails \psi$ for some formula $\psi$. The semantic clauses for these cases refer to the relations $\Rsigns$ and $\Rentails$ in way that treats
$u$ syntactically, without further decomposition that would result in a reinvocation of the semantic clause for $\this x. \phi$. The recursion is therefore not vicious.  

More formally, define the \emph{semantic tree} for a satisfaction expression  $M,w \models \phi$ to be the tree with nodes labelled by expressions 
of the form $M, w'\models \psi$, that has root labelled $M,w \models \phi$, and in which a node labelled by satisfaction expression $M, w'\models \psi$ has as children a node for  
each  satisfaction expression called recursively by the  definition of satisfaction (i.e., that occurs on the right hand side of the rule for $M, w'\models \psi$). 
The leaves of such a tree are the nodes labelled by a satisfaction expression that makes no recursive calls, i.e., the cases 
for sentences $p$ and $A \signs t$. The following shows that the recursion defining  $M,w \models \phi$ is well-founded. 

\begin{proposition} \label{prop:finheight}
For every sentence $\phi$, model $M$ and world $w$, the semantic tree for $M,w\models \phi$ has finite height. 
\end{proposition}

This result critically uses the fact that variables $x$ appear only in syntactic positions in formulas. 
A serious problem for the semantics would arise if we were to allow $x$ to occur more generally. For example, $\this x. \neg x$ is essentially the famous ``Liar Paradox'' \cite{sep-liar-paradox}, 
since it effectively states ``This formula is false". Applying the above semantics would yield $M, w\models \this x. \neg x$ iff  $M, w\models \neg (\this x. \neg x)$
iff not $M, w\models \this x. \neg x$, 
making the semantics itself inconsistent! 

 In an effort to give the most general possible solution to the Liar, for languages containing a truth predicate, 
a variety of approaches have been proposed, including hierarchies of languages and meta-languages \cite{Tarski}, 
fixed point semantics \cite{Kripke75}, or use of non-standard set theories \cite{TheLiar}. The scope of these approaches
is significantly beyond our needs, since our logic has no truth operator. 
One could attempt to follow the $\mu$-calculus \cite{Kozen83} and require that occurrences of $x$ inside $\phi$ must be in positive position for $\this x. \phi$ to be well-formed. 
We have not pursued such approaches here because we deliberately wish to treat $x$ semantically as a term, i.e., a piece of syntax, rather than as a property, as in the 
$\mu$-calculus. 

Having introduced the new self-reference construct with the above semantics, we get a new axiom for the logic: 

\begin{enumerate}[label=\textbf{Ax\arabic*}]
 \setcounter{enumi}{7}
\item \label{ax:self} $(\this x. \phi) \dimp \phi[x\mapsto (\this x. \phi)]$
\end{enumerate}

\begin{proposition} \ref{ax:self} is valid. 
\end{proposition} 

\begin{proof} 
Direct from the semantics. 
\end{proof}

\section{Application of self-reference to counterpart signatures} \label{cpt2}

In Section~\ref{cpt1}, we already gave the structure of the argument that individually signed copies of a contract $c$ 
imply the agents' assent to the logical content $\phi$ of the contract. That argument assumed that $c$ satisfies 
the formula $c \entails (((A \signs c) \land (B \signs c)) \rimp \phi)$. We now show that the syntax and semantics for self-reference 
developed above enables us to display a particular contract $c$ for which this formula is indeed a \emph{validity} of the logic. 
For the remainder of this section, let $c$ be the formula $$\this x.(((A \signs x) \land (B \signs x)) \rimp \phi)~.$$ 
Intuitively, this expresses ``This contract may be signed in counterparts" as 
``This formula, if signed by both $A$ and $B$, implies that $\phi$ holds'' where $\phi$ expresses the 
logical content of the contract.

\begin{proposition}  $\vdash c \entails (((A \signs c) \land (B \signs c)) \rimp \phi)$.  
\end{proposition} 

\begin{proof} 
We have the following instance of axiom~\ref{ax:self}: 
$$ c \dimp (((A \signs c) \land (B \signs c)) \rimp \phi)~.$$
By~\ref{ax:fomulaentails}, we have that  $c \entails c$. 
Hence, using~\ref{ax:entailsclosure}, and~\ref{r:MP}, we derive
$c \entails (((A \signs c) \land (B \signs c)) \rimp \phi)$. 
\end{proof} 

It follows using the argument of Section~\ref{cpt1} 
that we can use the particular formula $c$ to implement signature by counterpart of a contract with logical content $\phi$.

\begin{proposition} $\vdash ((A\signs c) \land (B\signs c)) \rimp ((A \says \phi) \land (B \says \phi))$. 
\end{proposition} 

Thus, we have the concrete self-referential formula $c$ as one example that supports signature in counterparts in our logic. 
Other examples are easily generated. For example, it is clear that for contracts involving a larger number of parties $A_1, \ldots, A_n$, 
for the formula $m$ defined as 
$$\this x.\left(\left(\bigwedge_{i=1\ldots n} A_i \signs x\right) \rimp \phi\right)$$
we have 
$$\vdash  \left(\bigwedge_{i=1\ldots n} A_i \signs m \right)\rimp \left(\bigwedge_{i=1\ldots n} A_i \says \phi \right)~.$$ 

\section{Common Assent} \label{sec:common} 

We now consider the intuitive interpretation of our operators and present some additional consequences of the semantics that follow from condition SC3. 
As we have noted, this condition underpins axiom~\ref{ax:liftsigns}, 
which states that 
$$  (B \signs t) \rimp A \says (B\signs t) $$
for all agents $A,B$. 

Whether this axiom is desirable is application dependent, and we do not propose that the semantics considered in the present paper is adequate for all applications. 
It is part of our intended interpretation of $B\signs t$ that agent $B$ has placed their  
non-repudiable and publicly verifiable (cryptographic) signature on the content $t$. This means that, presented with the signed content, 
no agent can reasonably dispute that $B \signs t$. 
However, it might be objected that $B\signs t$ does not imply that agent $A$ \emph{knows} that $B$ has signed $t$, since 
$A$ may not have seen the signed content.  We have not modeled knowledge in the logic, but if what $A$ assents to is based on 
$A$'s incomplete view of the world, then  there may well be true statements concerning what other agents have in fact signed about which $A$ is agnostic. 
In this event, validity of axiom~\ref{ax:liftsigns} would be stronger than is desirable. 
 
A stronger argument for the axiom can be made on the assumption that all signed statements 
are available to all agents. 
The intuitive motivation given earlier in the paper, that the semantics can be understood as modeling a scenario where 
all the cryptographic evidence is available, e.g., as in a court proceeding, 
supports this assumption.  

Another scenario that supports the axiom is a setting where a central trusted agent, such as a law firm or official registry, collects and stores all signed statements, 
and provides any such evidence to an agent upon request. Indeed, signature of contracts in counterparts often makes use of law firms for this purpose (see the discussion section below). 
With this assumption, a reading of $A \says \phi$ as  ``agent $A$ \emph{would} say $\phi$  once all the evidence has been obtained''  would support axiom~\ref{ax:liftsigns}.

A more secure way to realize such a scenario, particularly if there are questions about the trustworthiness or reliability of a third party, 
would be to eliminate use of a third party by using a blockchain to record the signed statements. (Blockchains use a variety of byzantine consensus protocols to implement an immutable ledger \cite{Nakamoto2008,AndroulakiCCSV17}.)  
In such an application $A \signs t$ could be taken to have the semantics that not only has $A$ cryptographically signed $t$, but that the signed copy of $t$ has been recorded on the 
blockchain. Similarly $A \says \phi$ can be interpreted as meaning that cryptographic evidence entailing that $A$ assented to $\phi$ is present on the blockchain.%
\footnote{There are subtleties about finality  and the stability of the record that depend on the details of the consensus protocol in use by the blockchain. Some blockchains
have the property that facts may be unstable, though only with negligible probability. For our purposes here we treat this as equivalent to actual stability for practical purposes.} 
In such an interpretation, there is a strong case for the validity of axiom~\ref{ax:liftsigns}, since a secure public record is available to all agents. 

Assuming axiom~\ref{ax:liftsigns}, we can derive some further conclusions. 
When $G$ is a group (set) of agents, write $G \says \phi$ for the conjunction $\bigwedge_{A\in G} A \says \phi$, 
and inductively define $G \says^k \phi$, where $k\geq 1$ is a natural number, by 
 $G \says^1 \phi = G \says \phi$ and 
 $G \says^{k+1} \phi = G \says (G \says^k  \phi)$. Define the semantics of $G \says^\omega \phi$ by 
 \begin{itemize}
 \item $M,w \models G \says^\omega \phi$ if $M,w \models G \says^k\phi$ for all natural numbers $k \geq 1$.  
 \end{itemize} 

Intuitively, $G \says^\omega \phi$ states that the group $G$ is in mutual agreement concerning $\phi$. Not only does everyone in the group 
agree to $\phi$ (since $G \says \phi$), but everyone agrees that everyone agrees, i.e., $G \says (G \says \phi)$, and they furthermore agree 
that everyone agrees that everyone agrees, i.e., $G \says^3 \phi$, and so on. This notion is very similar to the well-know notion of common knowledge
from the literature on epistemic logic \cite{FHMVbook} with the exception that we do not have $(G \says \phi) \rimp \phi$ valid. 
As the following result shows, it is a normal operator satisfying an induction condition. 

\begin{proposition} \label{prop:saysomega}
The operator $G \says^\omega \phi$ satisfies the following for all models $M$: 
\begin{enumerate} 
\item if $M \models \phi$ then $M \models G\says^\omega \phi$,
\item $M \models ((G\says^\omega \phi)\land G\says^\omega (\phi\rimp \psi)) \rimp G \says^\omega \psi$,
\item  if  $M \models \phi \rimp G \says (\phi \land \psi)$ then 
$M \models \phi \rimp G \says^\omega \psi$,
\item $M \models G \says^\omega \phi \dimp G \says (\phi \land G\says^\omega \phi)$.
\end{enumerate}
\end{proposition}

This result justifies the following axiom and rules of inference for $G\says^\omega \phi$:
\begin{enumerate}[label=\textbf{Ax\arabic*}]
 \setcounter{enumi}{8}
\item \label{ax:omegafix} $ (G\says^\omega \phi)\dimp  G\says ( \phi \land G \says^\omega \psi)$
\end{enumerate} 
\begin{enumerate}[label=\textbf{R\arabic*}]
 \setcounter{enumi}{3}
\item \label{r:Nomega} $\vdash  \phi$ implies $\vdash G \says^\omega \phi$. 
\item \label{r:induction-omega} 
$\vdash \phi \rimp G \says (\phi \land \psi)$ implies  
$\vdash \phi \rimp G \says^\omega \psi$.
\end{enumerate} 
We note that we derive from the above that 
$$\vdash (G\says^\omega \phi)\land G\says^\omega (\phi\rimp \psi)) \rimp G \says^\omega \psi$$
so that the operator $G \says^\omega$ is normal. 

Using these axioms and rules, we can derive a stronger statement about the effect of 
signing a contract. Let $c$ be the contract from Section~\ref{cpt2}.
The conclusion of our characterization of signature in counterparts was that 
$$\vdash (A \signs c \land B \signs c) \rimp \{A,B\} \says \phi$$ 
where $\phi$ expressess the terms of the contract.  

We derive using axiom~\ref{ax:liftsigns} that 
$$\vdash (A \signs c \land B \signs c) \rimp \{A,B\} \says  (A \signs c \land B \signs c \land \phi)~. $$ 
Using \ref{r:Nomega}, we get that 
$\vdash (A \signs c \land B \signs c) \rimp \{A,B\} \says^\omega \phi$.

That is, it follows from the fact that both $A$ and $B$ have signed the contract not just that $\{A,B\} \says \phi$
(both assent to the terms of the contract), but that $\{A,B\} \says^\omega \phi$, i.e., the parties are in mutual agreement about the content of the 
contract: they are also agreed that they are agreed, they agree that that they agree that they agree, etc. 
Certainly this is a desirable conclusion - problems could arise if were to accept that it is possible that the parties 
have agreed to the contract but they are in disagreement about whether they have agreed - one could  envisage one of the agents litigating on the question of whether a valid contract 
has in fact been formed, in order to escape the contract. The desirability of mutual agreement may in fact underlie the historical process of gathering all parties in a single location for a signature
ceremony, since such a setting, with all parties observing each other signing the contract, establishes common knowledge concerning the parties agreement to the contract. 

Indeed, we can draw a further conclusion. 
Using axiom~\ref{ax:liftsigns} we can obtain that 
$$\vdash (A \signs c \land B \signs c) \rimp H \says  (A \signs c \land B \signs c) $$
for every group $H$ of agents. 
Using \ref{r:Nomega}, we get that 
$$\vdash (A \signs c \land B \signs c) \rimp H \says^\omega  (A \signs c \land B \signs c)$$ 
and consequently that 
$$\vdash (A \signs c \land B \signs c) \rimp H \says^\omega ( \{A,B\} \says^\omega \phi)~.$$ 
That is, if $A$ and $B$ have signed, then not just they,  
but in fact all of society is in mutual agreement that $A$ and $B$ have assented to the terms of the contract.  
This again could be considered desirable, from the point of view of societal enforcement of contracts. 
Under an interpretation of $A\signs c$ as implying that the signature has been logged on a public blockchain, 
this conclusion is consistent with the conception of the blockchain as representing the consensus of all participating agents. 

On the other hand, it is reasonable that agents might be entitled to privacy concerning their contracts unless these come into dispute, 
which would argue against the reasonableness of this conclusion. However, if we  interpret $A \says \phi$ conditionally, as asserting that $A$ would agree to $\phi$ \emph{were} $A$ to be 
be presented with all the relevant (cryptographic) signature evidence in existence, then  we do not have that $A \says \phi$ implies $A$ knows that $A \says \phi$, and the conclusion is 
more reasonable.  We expect this intuition could be formalized by adding conditional, temporal and/or epistemic expressiveness to the framework, but leave this for future work. 

We remark that a similar argument to the above yields from axiom~\ref{ax:liftentails} that for all groups $H$ we have 
$$ (t\entails \phi) \rimp H \says^\omega (t \entails \phi) $$
Intuitively, this states that all agents in $H$ mutually agree to the entailments of a term $t$. This is exactly as we would expect, on the 
assumption that all agents ``speak the same language'' which moreover is common knowledge.

\section{Imperative Smart Contracts} \label{sec:imperativesc}

Our discussion above has been motivated by a view of smart legal contracts in which parties enter into relationships recognized 
by the legal system by signing declarative content expressed in a logic. 
In this section we consider how the logic of the present paper might relate to an imperative smart contracts. We first review how existing smart contract approaches relate to the law, 
and then discuss how the logic of the present paper might be used to model such approaches. 

\subsection{Current Smart Contract Approaches} 

Present smart contract platforms are generally built for a \emph{Smart Contract  Absolutist} or ``Code is Law" view that does not 
recognize legal jurisdiction, and uses imperative code rather than declarative representations of the  relationships enforced between the parties. 
While smart contract platforms such as Ethereum \cite{ethereum} use digitally signed messages, 
agreement to  contract terms is implicit. In effect, an offeror makes an offer by signing a transaction that registers code on to the blockchain. 
Other parties accept to engage on the terms enforced by this code not by means of messages that explicitly agree to these terms,
but simply by sending a signed transaction that calls a function of the on-chain code. Typically this ``accepting'' transaction also  transfers 
some asset (cryptocurrency or token) from control of the acceptor to the control of the on-chain code. It is this grant of control rather than an explicit agreement that
commits the acceptor to the terms of the contract. 

One step closer to our view is Digital Asset's DAML  smart contract language \cite{daml}. This language was developed to support smart contracts that 
do carry legal recognition, and are required to be legally compliant, for
applications such as equity rights representations on the digital ledger based clearance and settlement system under development for the Australian Stock Exchange \cite{asxchess}. 
Like Ethereum contracts, DAML contracts are code, but it is intended that this 
code may represent rights and obligations enforced either on-chain or in the real world. (Off-chain obligations are expressed in on-chain code simply as text fields.) 
Each DAML smart contract has ``signatories'', a set of parties. All signatories need to authorize a DAML contract before it can be can be registered on the blockchain. 
The authorizers of an action (a function call on a smart contract) are the agent(s) calling the action and the signatories of the contract on which it is called. 
To create a contract, all its signatories must be in the set of authorizers of the creation action. This view enables a two-party offer and acceptance process
in which the offeror first lodges  on-chain an offer contract for which it is the sole signatory, on which the acceptor calls an acceptance function that has the effect of creating the 
actual two-party contract. For $n$-party contracts, the DAML manual recommends a process using a sequence of on-chain contracts, starting with an initial contract 
with a single signatory, and adding the remaining signatories one-by-one through function calls that create the next contract in the chain. The call made by the final party creates the 
intended contract on chain, with all $n$ parties as signatories. This process is structured somewhat like the linear $n$-party offer-acceptance process discussed in Section~\ref{sec:offer}, but with the 
intended terms represented implicitly in the code.

A disadvantage of imperative code-based smart contracts on a Smart Contract  Absolutist view is that the possible behaviors of code may 
be difficult to understand, even for its original developers. Divergences from the expected behaviors 
sometimes have serious consequences, e.g., in case of malicious attacks exploiting design flaws \cite{DAO-NYT}, and leave the parties without legal recourse. 

Riccardian Contracts \cite{Grigg04} have been proposed as an approach to dealing with comprehensibility and the lack of legal recourse: the key idea is that parties to a contract digitally 
sign content that is expressed in a restricted form that can be interpreted both as a legal contract and processed by a machine. 
Due to this dual purpose, code in these contracts typically has a significantly more limited expressive power than the ``Turing-complete'' smart contracts on platforms like Ethereum. 
In some incarnations \cite{OpenBazaar}, Riccardian contracts are a set of attribute-value pairs. The attributes and values may be 
natural language words for readability, but their interpretation is external to the representation.
A Riccardian contract may also express (in natural language) the legal agreement between the parties as to how code will be interpreted by them
from a legal perspective. 

\subsection{Application of the Logic to Imperative Smart Contracts} 

We now sketch some of the ways that an extension of the logic of the present paper might be used to give a logical account of 
the way that legal meaning might be ascribed to imperative smart contracts. Our brief presentation will necessarily be incomplete. Because 
of the dynamic nature of the blockchain state and the relevance of questions of timing and order, 
extensions of the logic to encompass quantification and temporal and deontic expressiveness 
would be needed to give a full account. Such extensions are beyond the scope of the present paper, and 
we leave their development to future work. 

One of the ways that legal meaning might be associated to an imperative smart contract is simply to have the parties to the smart contract sign a legal contract 
(separately and distinct from the smart contract) that concerns the legal interpretation and consequences of the blockchain state of the smart contract 
and the messages signed by the parties. 
Such contracts would be similar to the ``trading partner agreements" that have long been used in Electronic Data Interchange (EDI) systems \cite{BaumPerritt}. 
For contracts of this nature, when their content is formalized as a formula $\phi$, either the  formalization of the offer and acceptance process from Section~\ref{sec:offer} or the 
signature in counterparts process from Section~\ref{cpt2} could be used to capture the way that the parties come to mutual assent of this legal contract. 
The only novelty is that the content of $\phi$ relates to the interpretation of a distinct smart contract.

Alternately, the legal system itself could choose to accept smart contracts and signed blockchain transaction messages as legally meaningful, and lay out a standardized legal interpretation. Such a move would be similar to the existing legislation that governs legal acceptance of electronic signatures
(e.g., the US Electronic Signatures in Global and National Commerce Act, and the EU eIDAS (electronic IDentification, Authentication and trust Services) regulation).    
We sketch one way that such an interpretation could be established, using the entailment operator $\entails$ to express the meaning 
associated to a signed transaction. 

We make a number of assumptions about the smart contract platform and the interpretation of some of the primitives of the logic: 
\begin{itemize} 
\item A smart contract $c$ is created by a party $A$ by signing a message of the form $\create(c)$. The statement $A\signs \create(c)$ is interpreted to be 
true if $A$ has cryptographically signed the message  $\create(c)$, and this cryptographically signed message has been processed by the miners, 
so that the contract $c$ has been registered on the blockchain.  
We assume that $c$ determines a description of a unique ``address" where the contract resides.%
\footnote{Something close to this holds in Ethereum, where the address for a contract creation is determined from the sender's address and a nonce that is included in the signed transaction.
A precise modelling of Ethereum would require that a function call to a created contract be directed to this address rather than to $c$. This would not be difficult, but we 
avoid this complication for brevity.  
} 
 A consequence of this assumption is that for a given $c$, 
there is at most one agent $A$ for which $A \signs \create(c)$ holds, since miners would have rejected any attempt by another agent to create a contract at the same address
where one already resides.  

\item A function call on a smart contract $c$ that exists on the blockchain is performed by sending a message of the form $\call(c,f)$ to the miners, 
where, as just noted, we interpret $c$ as including the address of the contract. 
Here $f$ specifies the function being called, 
as well as its arguments, any cryptocurrency value attached to the call, and any other information such as a nonce. 
We also assume that $A \signs \call(c,f)$ is interpreted to mean that $A$ has signed this message, and it has been processed by the miners and its effects
reflected in the blockchain state.  

\item In addition to the above, the parties may sign messages of the form $\phi$, where $\phi$ is a formula. For such messages, we do not necessarily require that
$A \signs \phi$ implies that the signature has been registered on the blockchain. Agents may sign such messages and transmit them privately to others that they 
are interacting with, or publish such a  signed message on a website. 

\item The relation $\entails$ expresses the meaning that is associated to signed messages relating to smart contracts within a particular jurisdiction. 
In general,  given the international scope of open blockchain systems, there would be more than one potential jurisdiction, which may vary in the interpretation
of smart contracts. To reason about such situations would require an extension of the present logic that relativizes the relation $\entails$ to a jurisdiction. This would be 
of interest for reasoning about cross-jurisdictional issues, but we leave this as an issue for future work. 

\end{itemize} 

Based on these assumptions, we can give a formalization of how two agents $A$ and $B$ can be interpreted as entering into a meeting of the minds when 
$A$ creates a smart contract $c$ and $B$ participates in this smart contract by sending it a signed transaction, making function call $f$ on this smart contract. 
Consider the formula schema
\begin{equation} 
\begin{array}{l}
(A \signs \create(c) \land A \signs \psi(c)) \rimp \psi(c) 
\end{array} \label{eq:smartcontract} 
\end{equation} 
This formula states the general principle that if $A$ has created $c$, and 
$A$ has signed a formula $\psi(c)$ expressing the legal interpretation of $c$, 
then that legal interpretation holds. That is, the formula expresses that the statements 
signed by the creator of the smart contract $c$ about the meaning of 
the smart contract determine the meaning of the smart contract. (In general, this is a very strong statement, and 
we would want to temper its force by placing restrictions on the formula $\psi(c)$, such as that $\psi(c)$ does not 
concern matters that are ``unrelated'' to $A$. We will see, however, that we use this formula only in the scope of 
operators ``$i \says$'', so that agents have a choice of whether or not to agree to the principle, given the 
particular formulas $\psi(c)$ that $A$ has signed.) 

For example, suppose that $A$ is Acme Co and $c$ is a smart contract that records share ownership in Acme Co. 
Here $\psi(c)$ might contain statements about the legal interpretation of aspects of the state of contract $c$, 
such as 
\begin{quote} 
$c.\mathit{shares}[B]=n$ $\rimp$ $B$ owns $n$ shares of Acme Co
\end{quote} 
where $\mathit{shares}$ is a variable of the smart contract $c$ of type mapping, that records an integer number of shares for each agent $B$. 
Contract calls might also be interpreted by including in $\psi(c)$ statements of the form 
$$
(B \signs \call(c,f)) \rimp \gamma(A,B,c,f)
$$
that associate a particular legal interpretation $\gamma(A,B,c,f)$ with the contract call $f$. 
For example, if  the function call $f$ 
is ``{\tt buy\{value:10~ether\}(20)}" then $\gamma(B,c,f)$ may express a legally meaningful assertion such as 
``immediately after completion of the function call, the caller $B$ owns 20 additional shares of Acme Co, and 10 ether has been transferred from B to Acme Co". 
Note that these formulas link the state of the blockchain to the domain of legal ownership within the governing jurisdiction: 
blockchain data is being \emph{interpreted} as having a particular meaning in the legal world.  

Write $\Phi(c, \psi(c))$ for formula~(\ref{eq:smartcontract}). Suppose now that the entailment relation has been defined
so that 
$$ 
\call(c,f) \entails \Phi(c,\psi(c)) ~.$$
for all $c,f$ and $\psi(c)$.
That is, the smart contract call message $\call(c,f)$  entails
the principle $\Phi(c,\psi(c))$. Intuitively, this means that both a caller of the smart contract assents
to the principle $\Phi(c,\psi(c))$ in signing the contract call. 

We claim that the above definition of the meaning of the smart contract call implies, according to our logic, 
that if $A \signs \create(c)$ and $A \signs \psi(c)$ and $B \signs \call(c,f)$, 
then $A\says \psi(c)$ and $B\says\psi(c)$. That is, it is a consequence of the above assumptions that both the creator $A$ and the 
caller $B$ of the smart contract assent to the consequences of the conditions $\psi(c)$ that 
$A$ has asserted about the meaning of participation in the smart contract.

That $A \says \psi(c)$ is immediate from $A \signs \psi(c)$ using Axiom~\ref{ax:signssays}. 
It follows from the facts about signatures, Axiom~\ref{ax:liftsigns} and normality of $\says$ that 
$$B \says ((A \signs \create(c)) \land A \signs \psi(c))$$ 
(in fact, this holds for all agents, not just $B$). 
From $B \signs \call(c,f)$ and $\call(c,f) \entails \Phi(c,\psi(c))$  we get $B \says \Phi(c,\psi(c))$ using Axiom~\ref{ax:signssays}. 
It now follows using normality of $\says$ that $B\says\psi(c)$.
Hence we have both $A \says \psi(c)$  and $B\says\psi(c)$, as claimed. 

There is one issue with the above approach, which is that $A$ may sign multiple formulas
stating alternate interpretations of the smart contract. A smart contract participant 
$B$ is at risk that $A$ will, in the event of a legal dispute concerning the smart contract $c$, present a signed statement 
$A \signs \psi'(c)$ in court when $B$ had participated in the smart contract, signing $\call(c,f)$ 
on the understanding that it would be interpreted using  $A \signs \psi(c)$. 

Various approaches might be used to give $B$ assurance as to which legal interpretation will apply. One is 
to require that $A$'s signature on $\psi(c)$ be registered on the blockchain, and to state clear rules 
for which of these statements applies in case of conflicting alternatives.  
Alternately, we could preempt the potential for conflicts by  requiring the smart contract creation call to have the form
$\create(c,\psi(c))$, where the legal interpretation is registered together with the smart contract at the time of creation. 
The principle~(\ref{eq:smartcontract}) would then be stated as 
$$(A \signs \create(c,\psi(c)) \rimp \psi(c)~.$$
Alternately $\psi(c)$ could be included as part of the text of the contract $c$ itself. 
This would be similar to the approach taken in Riccardian contracts.

\section{Related Work} \label{sec:related} 

We have focussed exclusively on reasoning about contract signatures. A fuller treatment of the meaning of the logical content of contracts requires a richer language with 
additional expressiveness covering time, actions, and deontic notions. This will require a correspondingly richer semantics than that of the present paper. 
There exist works on logical representation of contracts that attempt to support such a richer expressiveness, e.g, \cite{Kimbrough,daskalopulu-thesis}, 
but the signature process does not appear to have been considered. 

The operator $A \says \phi$ is similar to the operator $A \asays  \phi$ from access control and authentication logics, which have been surveyed  by Abadi \cite{Abadi08}. 
However, these logics generally do not have our distinction 
(critical to keeping our semantics of self-reference simple) between $A \says \phi$ and $ A \signs \phi$, where in the 
latter $\phi$ is treated as a syntactic term rather than as a proposition. 
(Exceptions include \cite{HalpernM01}, and, in the different context of logics for 
electronic commerce messaging, a body of work by Kimbrough \cite{Kimbrough} and others that uses  ``disquotation''
of a syntactic representation of messages, that can be understood as following a ``syntactic substitution'' treatment of 
modalities.)  

The need for an axiom 
$ B \asays  \phi \rimp A \asays  ( B \asays  \phi) $
similar to \ref{ax:liftsigns} is generally accepted in  access control and authentication logics.  
Indeed some logics in this class accept the much stronger axiom 
$ \phi \rimp A \asays  \phi$
although the basis of the logic in this case is generally taken to be intuitionistic, to avoid some undesirable consequences in a classical setting.  
The motivation for such axioms in the 
context of access control logic is to  obtain validities such as 
$$ (A \asays  ((B\asays  \phi) \rimp \phi)) \land (B \asays  \phi) \rimp  A \asays   \phi $$ 
which enables $A$ to delegate to $B$ the ability to ``speak for'' $A$ on $\phi$, 
by  $A \asays  ((B \asays  \phi) \rimp \phi)$. 
We note that this motivation is very similar to our account of offer and acceptance in Section~\ref{sec:offer}.

There exists a body of work in the cryptography literature on 
``contract signing protocols'' or ``fair exchange protocols''
 \cite{PagniaVG03,KremerMZ02}. 
A protocol is said to be \emph{fair} if it ensures that the parties receive fully signed copies of the 
contract atomically, i.e., neither party has a fully signed copy until it is guaranteed that the other will also 
obtain a copy. Some protocols also aim to be 
\emph{abuse-free} \cite{GarayJM99}, in the sense that neither party is ever in a position where they have not yet assented to the contract, 
but are able to prove to a third party that they unilaterally have the ability to produce a fully signed copy (enforcing assent of the other party). 
Some general impossibility theorems  imply that it is often not possible to 
achieve fairness without use of a trusted third party \cite{EY80}, but protocols may attempt to minimize the use of this
 third party in various ways, e.g., using them as a fallback in case one party attempts to cheat the other \cite{AsokanSW00}. 
Some recent work has sought to use blockchain as the basis for fair exchange protocols, effectively decentralizing the trusted third party. 
A general construction for fair computation in the setting of Bitcoin is given in \cite{BentovK14}. More specific protocols focussed on contract signing are developed in \cite{Ferrer-GomilaHI19,WangLLZX19,ZZYX20}. 
Some of these rely on a redefinition of signature that meets one of our proposed interpretations of $A \signs t$, 
e.g., the protocol in \cite{Ferrer-GomilaHI19} does not consider a message to be signed until it has been 
registered on the blockchain.  
 
The fair exchange problem is orthogonal to the issues we have addressed in the present paper.  We are concerned with the semantics of the individually signed messages, 
and, in effect, reason in the final state of a fair exchange, where these messages have been successfully exchanged.  We do not address the question of abuse-freedom: 
even if it does not constitute assent to the contract, 
agent $A$'s signature on the self-referential formula $c$ we have developed could very well be sufficient evidence 
for a third party of $A$'s willingness to engage in the contract. Moreover, we have assumed it is sufficient for validity of the contract simply that 
$A\signs c$ and $B \signs c$, without considering the issue of who possesses the cryptographic evidence. In Section~\ref{sec:common} 
we argued that the use of a trusted third party or blockchain best justifies some aspects of our semantics for $\says$. 
A potential topic for future research is the interaction between our semantic viewpoint on messages and fair exchange protocols: it may be possible to 
develop a declarative understanding of the intermediate messages in these protocols, in the spirit of attempts to give a declarative meaning to messages in cryptographic protocols 
such as authentication protocols \cite{BAN90}. A richer modeling incorporating temporal and epistemic dimensions 
would be appropriate for such a project, and, depending on the nature of message passing environment and blockchain protocol, the appropriate notion of common knowledge 
may well be a more complex form of fixed point \cite{HalpernP17}.

In the present paper, we have been primarily concerned with developing a logical understanding of processes for contract signature as it relates to a meeting of the minds. Beyond this issue, 
there are several concerns relating to the signature process that affect the legal standing of the contract. Contracts often need to be not just signed but also given official standing as  a 
contract by being `sealed' (a term that derives from the historical use of wax seals for this purpose - nowadays a signature may serve the same purpose). 
Legislation affecting particular types of contracts may place additional requirements, e.g., use of witnesses, and registration of the  contract with a registrar, 
who may impose particular physical forms on the contract, such as original signed copies or specific types and sizes of paper. 

A report \cite{signatures-report} by a group of major law firms has developed general principles  and three distinct protocols for remote signing  of financial documents. 
Generally, these require signers to print and sign a paper copy of the contract and/or signature page, but allows scanned copies of these to be returned. 
All the protocols assume a coordinating legal law firm, so they use a centralized trusted third party. Statements made by the signatories in the emails by which the scanned
copies are delivered address the questions of sealing and validity date of the contract. The report does not go into the general legal principles or security requirements underlying the design of these
protocols, or elucidate how the protocols meet the requirements for the particular types of contract for which they are recommended. It may be interesting to pursue these questions  in future work, 
using a formal methodology similar to that of the present paper.

\section{Conclusion} \label{sec:concl} 

We have argued in this paper for some particular syntactical interpretations of the meaning of signing a contract in various processes. 
We have given natural axioms for the logical operators used, and demonstrated that these axioms justify reasoning steps that show 
that these processes satisfy a criterion of `meeting of the minds'.  Our formal semantics in this paper has been constructed as a simple semantics that validates the axioms. 

We note that  questions about the right semantics  do not affect the main conclusions of the paper in Section~\ref{cpt1} and Section~\ref{cpt2}. We have established these 
conclusions proof theoretically, using only the minimal set of rules and axioms in Section~\ref{sec:logic} and Section~\ref{sec:self-ref}, 
so these conclusions should be acceptable to anyone who accepts 
the correctness of these rules and axioms. 

We have not attempted to prove a completeness result for our logic, but have merely developed a semantics that 
validates the axioms we have chosen to work with. This suffices to 
show that the logic is consistent, and was useful in Section~\ref{cpt1} to show that a particular entailment does not hold. 
We do not expect that there are inherent difficulties in proving a completeness result, but defer this to future work on a richer 
logic. 

For some applications, e.g., asynchronous message passing contexts where agents do not have access to a common source of truth about what has been signed, 
the conclusions of  Section~\ref{sec:common} may be considered to be too strong, and it may be desirable to move to a weaker 
semantics that drops semantic condition SC3 and the corresponding axiom~\ref{ax:liftsigns}. In such settings, it would be beneficial to introduce 
an operator that expresses that an agent ``has" a message.  We leave for future work the question of what, from the point of view of intuitive acceptability and the needs of applications, 
are the appropriate axioms beyond the ones we have used,  as well the question of what more liberal semantics 
supports the required axioms. 

A general issue for contracts is that the legal system, either through legislation or court rulings, may make determinations concerning the interpretation of a contract
that are at variance with its text. For contracts in paper form, court rulings can be accommodated by reversing actions, payment of  compensation or replacement of 
the contract by another. In the setting of immutable contracts and events on a blockchain, these types of accommodation may be more difficult. 
There is a recognised need for smart contracts representing legally meaningful content to 
be adaptable to legal rulings \cite{MarinoJ16}. How this issue relates to the logical view we have taken in this work we also leave as a topic for future research.

\bibliographystyle{alpha} 
\bibliography{counterparts}

\appendix 

\section*{Appendix}

\setcounter{proposition}{8}

In this appendix, we give the proofs omitted in the body of the paper.

\begin{oldtheorem}{prop:sound}
\begin{proposition} 
The axiom schemas~\ref{ax:prop}-\ref{ax:liftentails} 
and  rules of inference~\ref{r:MP}-\ref{r:Nsays}
are valid in models satisfying SC1-SC3.
\end{proposition}
\end{oldtheorem}

\begin{proof} 
Axiom \ref{ax:prop} and rule \ref{r:MP} are immediate from the fact that the boolean operators in formulas have their usual semantics. 
Axiom \ref{ax:fomulaentails} is direct from SC1.
Axioms \ref{ax:entailsclosure}, \ref{ax:saysnormal} and rules \ref{r:Nentails} and \ref{r:Nsays} follow in the usual way from the fact that the 
operators $t \entails \phi$ and $A \says \phi$ have been given a standard Kripke semantics using relations $\Rentails$  and $\Rsays$.

Axiom \ref{ax:signssays} follows from SC2. 
For, suppose $M,w \models (A \signs t)\land (t \entails \phi)$. 
From $M,w \models (A \signs t)$ we have that $(w,A,t) \in \Rsigns$. 
Let $w'\in W$ be any world such that $(w,A,w')\in  \Rsays $. By SC2, we have  $(t,w') \in \Rentails $.
Thus, from $M,w \models (t \entails \phi)$, we get $M,w' \models \phi$. 
We have shown that for all $w'\in W$ with $(w,A,w')\in  \Rsays $, we have $M,w' \models \phi$. 
Thus, $M,w \models A \says \phi$. 

For axiom ~\ref{ax:liftsigns}, suppose that $M,w \models B \signs t$. Then $(w,B,t) \in \Rsigns$. 
Let $w'\in W $ be any world with $(w,A,w') \in \Rsays$. 
By SC3, we have $(w',B,t) \in \Rsigns$.
Thus, $M,w'\models B \signs t$ for all $w'\in W$ with $(w,A,w') \in \Rsays$, which is equivalent to $M,w \models A \says (B \signs t)$. 

For axiom ~\ref{ax:liftentails}, suppose that $M,w \models  t\entails \phi$. 
Then $M,w' \models \phi$ for all $w'\in W$ such that $(t,w') \in \Rentails$. 
Note that this condition is independent of $w$. 
This means that it holds not just for $w$, but for every world $w''\in W$. 
In particular, it holds at every world $w''$ such that  $(w,A,w'') \in \Rsays$. 
Thus, $M,w'' \models t\entails \phi$ for all $w''\in W$ with $(w,A,w'') \in \Rsays$, 
which is equivalent to $M,w \models A \says (t\entails \phi)$. 
\end{proof}

For Proposition~\ref{prop:extend}, we proceed as follows. 

Define the entailment depth $\edepth$ of a term, inductively by 
$$\begin{array}{rl}  
\edepth(t) &  = 0 \hfill \text{when $t\in \terms \setminus \formulas$}\\ 
\edepth(p) & = 1 \\
\edepth(\neg \phi) & =\edepth(\phi) \\ 
\edepth(\phi_1 \land \phi_2) & = max(\edepth(\phi_1), \edepth(\phi_2)) \\ 
\edepth(A \signs t) & = 1 \\  
\edepth(t\entails \phi) & = max(\edepth(t), \edepth(\phi)) + 1 \\ 
\edepth(A \says \phi) & = \edepth(\phi) 
\end{array} 
$$
Note that terms that are not formulas have entailment depth 0, 
and formulas not containing $\entails$ have entailment depth 1. 
The formula $p \entails p$ has depth 2, and $(p\entails p) \entails p$ has depth 3, since $p \entails p$ is in $\formulas$.

For relations $R\subseteq \terms\times W$ and $R'\subseteq \terms\times W$, define $R \equiv_k R'$ when 
for all terms $t$ with $\edepth(t) \leq k$ and $w \in W$ we have $(t,w) \in R$ iff $(t,w) \in R'$. 
Intuitively, $R \equiv_0 R'$ when $R$ and $R'$ agree on the entailments of all terms that are not formulas,
$R \equiv_1 R'$ implies that, additionally, $R$ and $R'$ agree on the entailments of formulas that do not
contain $\entails$, and $R \equiv_2 R'$ implies that $R$ and $R'$ agree on the entailments of formulas that contain 
 $\entails$, but with a single depth of nesting, etc.

When $M = \langle W, \Rsigns, \Rentails, \Rsays , \pi \rangle$ is a model and $R \subseteq \terms\times W$ is a relation, we define
$M(R) = \langle W, \Rsigns, R, \Rsays , \pi \rangle$ to be the result of substituting $R$ for $\Rentails$.

\begin{proposition} \label{prop:Romega}
Let $k \geq 0$ and let $W$ be the set of worlds of a model $M$.  
Suppose that $R,R'\subseteq \terms\times W$ are relations such that $R \equiv_k R'$. 
Then for all $w\in W$ and formulas $\phi$ with $\edepth(\phi) \leq k+1$ 
we have $M(R),w \models  \phi$ iff $M(R'),w \models  \phi$. 
\end{proposition}

\begin{proof} 
By induction on $k$. For the base case $k = 0$, assume $R \equiv_0 R'$.
Formulas $\phi$ with depth $\edepth( \phi) \leq k+1 = 1$ 
do not contain the operator $\entails$, and their semantics does not depend on
the relation $\Rentails$ in a model. A straightforward induction on the construction 
of $\phi$ shows that for all $w\in W$, we have $M(R),w \models \phi$ iff 
$M(R'),w \models \phi$.  The details are similar to the argument for the 
inductive case given below, so omitted. 

For the inductive case, suppose by way of inductive hypothesis that 
$R \equiv_k R'$ implies that for all $w\in W$ and formulas $\phi$ with $\edepth(\phi) \leq k+1$ 
we have $M(R),w \models  \phi$ iff $M(R'),w \models  \phi$. 
Suppose $R \equiv_{k+1} R'$. Then also $R \equiv_{k} R'$.  By induction,  for all $w\in W$ and formulas $\phi$ with $\edepth(\phi) \leq k+1$ 
we have $M(R),w \models  \phi$ iff $M(R'),w \models  \phi$. 
We prove by a further induction on construction of $\phi$ that if $\edepth(\phi) \leq k+2$ then 
$M(R),w \models \phi$ iff $M(R'),w \models \phi$.  We have the following cases: 
\begin{itemize} 
\item $\phi$ is an atomic proposition $p\in \Prop$. Here $M(R),w \models \phi$ iff $p \in \pi(w)$ iff $M(R'),w \models \phi$.
\item $\phi = \neg \phi_1$. Here $\edepth(\phi_1) = \edepth(\phi) \leq k+2$, so by induction, 
$M(R),w \models \phi$ iff not $M(R),w \models \phi_1$ iff not $M(R'),w \models \phi_1$ iff $M(R'),w \models \phi$.
\item $\phi = \phi_1 \land \phi_2$. Here $\edepth(\phi_1), \edepth(\phi) \leq k+2$, so by induction, 
$M(R),w \models \phi$ iff $M(R),w \models \phi_1$ and $M(R),w \models \phi_2$
iff $M(R'),w \models \phi_1$ and $M(R'),w \models \phi_2$
iff  $M(R'),w \models \phi$.
\item $\phi = A \signs t$. Here $M(R),w \models \phi$ iff $(A,t) \in \Rsigns$ iff $M(R'),w \models \phi$. 
\item $\phi = t \entails \phi_1$. Here $t$ may be a formula, and we have 
$\edepth(t), \edepth(\phi_1) \leq k+1$. 
Since $R \equiv_{k+1} R'$ and $\edepth(t) \leq k+1$ we have $(t,u) \in R$ if $(t,u) \in R'$, for all $u \in W$. 
Since $\edepth(\phi_1) \leq k+1$ and  $R \equiv_{k} R'$, 
we have for all $u \in W$ that $M(R), u \models \phi_1$ iff $M(R'), u \models \phi_1$. 
Thus  $M(R),w \models \phi$ iff for all $u \in W$ with $(t,u) \in R$ we have $M(R), u \models \phi_1$
iff for all $u \in W$ with $(t,u) \in R'$ we have $M(R'), u \models \phi_1$ iff  $M(R'),w \models \phi$.
\item $\phi = A \says \phi_1$. Here $\edepth(\phi_1) = \edepth(\phi) \leq k+2$, 
and by induction, we have $M(R), w' \models \phi_1$ iff $M(R'), w' \models \phi_1$. 
Hence $M(R), w \models \phi$ iff 
for all $w'\in W$ we have $(w,A,w') \in \Rsays$ implies  $M(R),w'\models \phi_1$ 
iff for all $w'\in W$ we have $(w,A,w') \in \Rsays$ implies  $M(R'),w'\models \phi_1$ 
iff $M(R'),w\models \phi$. 
\end{itemize} 
\end{proof} 

We now show that, given a model $M = \langle W, \Rsigns, \Rentails, \Rsays , \pi \rangle$
and a relation $R^0 \subseteq (\terms \setminus \formulas) \times W$, 
we can construct a relation $R^\omega$ such that $R^0 \equiv_0 R^\omega$ and 
$M(R^\omega)$ satisfies constraint SC1. We obtain 
$R^\omega$ as the limit $\bigcup_{i<\omega} R^i$ of a sequence of relations 
$R^i \subseteq \terms \times W$, defined inductively by 
$$R^{i+1} = R^i \cup \{(\phi,w)~|~ \phi \in \formulas,~ \edepth(\phi) = i+1, ~ M(R^i),w\models \phi~\}~.$$ 
Intuitively, the following result states that in $M(R^\omega)$, the 
entailments of terms in $\terms \setminus \formulas$ are exactly as in $R^0$,  
and each formula entails just itself (plus anything that is valid in $M$). 

\begin{proposition} 
Let $M = \langle W, \Rsigns, \Rentails, \Rsays , \pi \rangle$ be a model and 
and let $R^0 \subseteq (\terms \setminus \formulas) \times W$ be a relation.  
Then $R^0 \equiv_0 R^\omega$ and for all $\phi \in \formulas$, and $w \in W$, we have 
$ (\phi,w) \in R^\omega$ iff $M(R^\omega), w \models \phi$. 
\end{proposition}  

\begin{proof} 
It is immediate from the construction and the fact that $\terms \setminus \formulas = \{t \in \terms ~|~\edepth(t) = 0\}$ that 
$R^0 \equiv_0 R^\omega$. Similarly, by construction, for all $k> 0$ we have $R^k \equiv_k R^\omega$.
Thus $R^k \equiv_k R^\omega$ for all $k \geq 0$. 
Thus, by Proposition~\ref{prop:Romega} we have, for all 
$k\geq 0$, that for all $w\in W$ and formulas $\phi$ with $\edepth(\phi) \leq k+1$ 
we have $M(R^\omega),w \models  \phi$ iff $M(R^k),w \models  \phi$. 
Since  for $\phi \in \formulas$ with $\edepth(\phi) = k+1$, 
we have $(\phi,w) \in R^{k+1}$ iff $M(R^k),w \models  \phi$, 
we have $(\phi,w) \in R^\omega$ iff 
$(\phi,w) \in R^{k+1}$ iff $M(R^k),w \models  \phi$ iff $M(R^\omega),w \models  \phi$.   
\end{proof} 

Note that we obtain, in particular, that $R^0 \equiv_0 R^\omega$ and $M(R^\omega)$ satisfies SC1. 
Proposition~\ref{prop:extend} is therefore a corollary of this result.

\begin{oldtheorem}{prop:finheight}
\begin{proposition} 
For every sentence $\phi$, model $M$ and world $w$, the semantic tree for $M,w\models \phi$ has finite height. 
\end{proposition} 
\end{oldtheorem} 

\begin{proof}
We write $\height(M,w\models \phi)$ for the height (possibly infinite) of the semantic tree for $M,w\models \phi$. 
We note that since the set of worlds is potentially infinite, nodes may have infinitely many children, so it does not necessarily 
hold that the semantic tree is finite.  We show the stronger proposition that for all models $M,M'$ with, respectively, worlds $w,w'$, 
and sentences $\phi$, we have $\height(M,w\models \phi)= \height(M',w'\models \phi)$ is finite. That is, the height of the semantic tree 
depends only on $\phi$, and not on $M$ or $w$.  

Define the \emph{semantic $\this$-depth} of a formula (not necessarily a sentence) $\phi$, denoted $\depth(\phi)$, 
to be the depth of nesting of \emph{semantic} occurrences of the operator $\this$ in $\phi$. 
This excludes (syntactic) occurrences of $\this$ in terms $t$ in subformulas of $\phi$ of the forms $A \signs t$ for some agent $A$, or $t\entails \psi$. 
More precisely, we define $\depth(\phi)$ inductively, by 
$$\begin{array}{rl}  
\depth(p) & = 0 \\
\depth(\neg \phi) & =\depth(\phi) \\ 
\depth(\phi_1 \land \phi_2) & = max(\depth(\phi_1), \depth(\phi_2)) \\ 
\depth(A \signs t) & = 0 \\  
\depth(t\entails \phi) & = \depth(\phi) \\ 
\depth(A \says \phi) & = \depth(\phi) \\ 
\depth(\this x (\phi)) & = \depth(\phi)+1 
\end{array} 
$$
We claim that for all formulas $\phi$ and terms $t$, we have $\depth(\phi[x\mapsto t]) = \depth(\phi)$. 
Intuitively, this is because in formulas $\phi$, the variable $x$ may occur only in syntactic positions, where it does 
not contribute to the depth. The proof is by induction on the construction of $\phi$. The cases for atomic propositions $p$ 
and formulas $A \signs u$ are trivial, and the cases for $\phi$ of the form $\neg \psi$, $\psi_1 \land \psi_2$ $t \entails \psi$ and $A \says \psi$ are 
straightforward, e.g.. 
\begin{align*} 
\depth((A \says \psi)[x \mapsto t]) & = \depth(A \says ( \psi[x \mapsto t]) \\ 
& = \depth( \psi[x \mapsto t]) \\ 
& = \depth( \psi) & \text{(by induction)} \\
&  = \depth(A\says\psi)~.
\end{align*} 
For $\phi = \this y (\psi)$, we have two cases. 
If $x = y$, then  
\begin{align*}
\depth( \phi[x\mapsto t]) & = \depth((\this y(\psi))[x \mapsto t]) \\
&  = \depth(\this y(\psi)) \\
& = \depth(\phi)~.
\end{align*}
If $x \neq y$, then 
\begin{align*} 
\depth( \phi[x\mapsto t])  & = \depth((\this y(\psi))[x \mapsto t]) \\ 
& = \depth((\this y(\psi[x \mapsto t])) \\ 
& = \depth( \psi[x \mapsto t]) + 1  \\ 
& = \depth( \psi) + 1 & \text{(by induction)}\\
& = \depth( \this y(\psi))\\
& = \depth(\phi)
\end{align*} 
as required. 

\newcommand{\ord}{k}
We can now prove the (generalized version of) the result. For a formula $\phi$ let $|\phi|$ be the size of $\phi$, i.e, the number of symbols in $\phi$. 
We proceed by induction using the well-founded order on formulas induced by the mapping $\ord : \phi \mapsto (\depth(\phi),|\phi|)$ 
from the lexicographic order on pairs of natural numbers, which is well-founded. 

Clearly, for all $M,w$,  we have  $\height(M,w\models p)=0$ and $\height(M,w \models A \signs t)=0$  since these nodes have no children. 

For $\neg \phi$, note that $k(\neg \phi) = (\depth(\neg\phi), |\neg \phi| ) = (\depth(\phi), |\neg \phi|) >  ((\depth(\phi), |\phi|) = k(\phi)$. 
Thus, we have $\height(M,w\models \neg \phi)= \height(M,w\models \phi) +1$ which is finite and independent of $M,w$ by induction.  
 
Similarly, for $ \phi_1 \land \phi_2$ we have 
\begin{align*} 
k(\phi_1 \land \phi_2) &  = (\depth(\phi_1 \land \phi_2), |\phi_1\land \phi_2|)\\
& =( \max(\depth(\phi_1), \depth(\phi_2)), |\phi_1 \land \phi_2|) \\
& > (\depth(\phi_i), |\phi_i|)\\
& = k(\phi_i)
\end{align*} 
for each $i= 1,2$. 
By induction, for each $i=1,2$, we have $\height(M',w' \models \phi_i)$ is finite and independent of $M',w'$.  
Hence $\height(M,w\models \phi_1 \land \phi_2)= max(\height(M,w\models \phi_1),\height(M,w\models \phi_2)) +1$ is also finite and independent of $M,w$. 

For nodes labelled $M,w \models t\entails \phi$ we have a child $M,w' \models \phi$ for each world $w'$ of $M$ such that $(t,w') \in \Rentails$. 
Note 
\begin{align*} 
k(t \entails \phi) & = (\depth(t\entails \phi), |t \entails \phi|)\\
&  = (\depth(\phi), |t \entails \phi|) \\
& > (\depth(\phi), |\phi|) \\
& = k(\phi) ~. 
\end{align*}
By induction, we have that $\height(M,w' \models \phi)$ is independent of $M,w'$ and finite. 
Hence we have  that $\height(M,w \models t \entails \phi) = \height(M,w \models \phi) + 1$ is also finite and independent of $M,w$. 
The argument for $M,w \models A \says \phi$ is similar. 

For nodes labelled  $M,w \models \this x ( \phi)$, we have one child, labelled 
$M,w \models \phi[x \mapsto \this x(\phi)]$, 
so $\height(M,w \models \this x ( \phi)) = \height(M,w \models \phi[x \mapsto \this x(\phi)]) + 1$. 
Using the fact, proved above, that  $\depth(\phi[x\mapsto t]) = \depth(\phi)$ for all terms $t$, 
we have in particular that  $\depth(\phi[x\mapsto \this x (\phi)]) = \depth(\phi)$. 
Hence 
\begin{align*} 
k(\this x ( \phi) ) & =  (\depth(\this x ( \phi)), |\this x (\phi)|) \\
&  = (\depth(\phi) + 1, |\this x (\phi)|) \\
& = (\depth(\phi[x \mapsto \this x (\phi)]) + 1, |\this x( \phi)|)\\
& > (\depth(\phi[x \mapsto \this x (\phi)]),|\phi[ x \mapsto \this x (\phi)]|) \\
& = \depth(\phi[ x\mapsto \this x (\phi)])~. 
\end{align*} 
Hence, by induction, we have that $\height(M,w \models \phi[x \mapsto \this x(\phi)])$
is finite and independent of $M,w$. 
It follows that $\height(M,w \models \this x ( \phi))$ is also finite and independent of $M,w$, 
as required. 
\end{proof} 

\begin{oldtheorem}{prop:saysomega}
\begin{proposition} 
The operator $G \says^\omega \phi$ satisfies the following for all models $M$: 
\begin{enumerate} 
\item if $M \models \phi$ then $M \models G\says^\omega \phi$,
\item $M \models ((G\says^\omega \phi)\land G\says^\omega (\phi\rimp \psi)) \rimp G \says^\omega \psi$,
\item  if  $M \models \phi \rimp G \says (\phi \land \psi)$ then 
$M \models \phi \rimp G \says^\omega \psi$,
\item $M \models G \says^\omega \phi \dimp G \says (\phi \land G\says^\omega \phi)$.
\end{enumerate}
\end{proposition} 
\end{oldtheorem}

\begin{proof} 
Properties (1) and (2) follow straightforwardly by induction from normality of the operator $A\says$ for all agents $A$. 

For (3), suppose $M \models \phi \rimp G \says (\phi \land \psi)$ . We show by induction on $k$ that 
$M \models \phi \rimp G\says^k (\phi \land \psi)$ for all $k\geq 1$. This yields (3) by using (2). 
The base case of $k=1$ is simply a restatement of the assumption. 
Assume $M \models \phi \rimp G\says^k (\phi \land \psi)$. 
By normality of $A \says$, we obtain 
$M \models (A\says \phi) \rimp A \says (G\says^k (\phi \land \psi))$ for all $A\in G$, 
and hence  
$M \models (G \says \phi) \rimp G \says (G\says^k (\phi \land \psi))$. 
Similarly, by normality, we obtain from the original assumption that $M \models \phi \rimp G \says \phi$. 
Thus $M \models \phi \rimp G\says^{k+1} (\phi \land \psi)$. 

For (4), note that $M,w \models G \says^\omega \phi$ implies  for all $k \geq 1$ that $M,w \models G \says^{k+1} \phi$, 
hence  $M,w \models G \says (G \says^{k} \phi)$, as well as $M,w \models G \says \phi$. 
Thus, for all $A\in G$ and $(w,A,w') \in \Rsays$ and $k \geq 1$, we have $M,w' \models \phi \land G\says^k \phi$. 
We obtain from this that $M,w \models G \says (\phi \land G \says^\omega \phi)$. Conversely, if  
$M,w \models G \says (\phi \land G \says^\omega \phi)$, we have $M,w \models G \says \phi$ and 
$M, w \models G \says (G \says^k \phi)$, for all $k \geq 1$, i.e., $M, w \models G \says^k \phi$, for all $k \geq 2$. 
Thus, $M,w \models G \says^\omega \phi$. 
\end{proof}

\end{document}